\documentclass[letterpaper, 10 pt, conference]{ieeeconf}  

\IEEEoverridecommandlockouts      

\usepackage{amsmath} 
\usepackage{graphicx}
\usepackage{booktabs}
\usepackage{amssymb}
\usepackage{siunitx}
\usepackage[english]{babel}
\usepackage[noadjust]{cite}
\usepackage{units}
\usepackage{epstopdf}
\usepackage[font=small]{caption}
\usepackage{subcaption}
\usepackage{comment}
\usepackage{amsfonts}
\usepackage{bm}
\usepackage{algorithm}

\usepackage{amsthm}
\usepackage{algorithmic}

\usepackage{color}

\usepackage{xcolor}
\newcommand{\remove}[1]{}

\DeclareMathOperator*{\argmin}{arg\,min}

\theoremstyle{plain}

\newtheorem{theorem}{Theorem}

\newtheorem{definition}[theorem]{Definition}

\newtheorem{remark}{Remark}

\title{\LARGE \bf
Probabilistic Safety-Assured Adaptive Merging Control \\for Autonomous Vehicles
}

\author{Yiwei Lyu$^1$, Wenhao Luo$^2$ and John M. Dolan$^3$
\thanks{$^*$This work was supported by the CMU Argo AI Center for Autonomous
Vehicle Research.}
\thanks{$^1$The authors are with the Department of Electrical and Computer Engineering, Carnegie Mellon University, Pittsburgh, PA, 15213 USA. Email: {\tt \small yiweilyu@andrew.cmu.edu}}%
\thanks{$^2$The authors are with the Robotics Institute, Carnegie Mellon University, Pittsburgh, PA 15213 USA. Email: {\tt \small \{wenhaol, jmd\}@cs.cmu.edu}}%
}

\begin{document}

\maketitle
\thispagestyle{empty}
\pagestyle{empty}

\begin{abstract}

Autonomous vehicles face tremendous challenges while interacting with human drivers in different kinds of scenarios. Developing control methods with safety guarantees while performing interactions with uncertainty is an ongoing research goal. In this paper, we present a real-time safe control framework using bi-level optimization with Control Barrier Function (CBF) that enables an autonomous ego vehicle to interact with human-driven cars in ramp merging scenarios with a consistent safety guarantee. In order to explicitly address motion uncertainty, we propose a novel extension of control barrier functions to a probabilistic setting with provable chance-constrained safety and analyze the feasibility of our control design. The formulated bi-level optimization framework entails first choosing the ego vehicle's optimal driving style in terms of safety and primary objective, and then minimally modifying a nominal controller in the context of quadratic programming subject to the probabilistic safety constraints. This allows for adaptation to different driving strategies with a formally provable feasibility guarantee for the ego vehicle's safe controller. Experimental results are provided to demonstrate the effectiveness of our proposed approach.

\end{abstract}

\section{Introduction}

Given the fact that self-driving cars (or autonomous vehicles (AV)) will not immediately replace all human-driven cars, they will have to share roads with human drivers for a long time. In the area of autonomous vehicle control and planning \cite{gu2015tunable, Ames2014, luo2016distributed}, safety is always the primary focus. The safety requirement becomes more strict as the scenario complexity increases. Ramp merging is a typical scenario where autonomous vehicles interact with human-driven cars. Human drivers introduce uncertainty into merge scenarios, and AV will collide if they can't plan and be controlled safely. 

Current control methods are not able to guarantee safety and efficiency at the same time. Traditional Automated Cruise Control (ACC)-like distance control methods force the vehicle to brake when distance is less than the specified minimum safety distance, while maintaining the vehicle at a desired driving speed. However, the National Highway Traffic Safety Administration (NHTSA) has reported that these kinds of methods sometimes lead to abrupt or aggressive behavior in response to a lead vehicle's velocity change, which can be very dangerous, especially in the highway ramp merging scenario \cite{nhtsa-2015}. In addition, vehicle control with ACC in the real world can be difficult, since sometimes its objectives can conflict with each other \cite{7040372}. NHTSA categorizes these methods as convenience features, rather than safety features \cite{nhtsa-2015}.

CBF-based methods \cite{7782377, inproceedings, luo2020multi, 9196757, zeng2020safetycritical, 9029446, 7040372, 7857061, 8917473} have become increasingly popular in the control application domain due to their forward-invariant property, which can provide a safety guarantee. However, the forward-invariant property relies on the solution feasibility, i.e., as long as a control solution satisfying the CBF constraints can be found at each time step, then the safety for future time steps can always be guaranteed. In reality, we may not always find such a solution and alternative solutions include switching to a full braking mode \cite{7857061, celi2019deconfliction}. However,
this could lead to serious consequences in autonomous driving scenarios, e.g., when there is another vehicle behind the ego vehicle, or it is too late for the AVs to stop. Therefore, systematic solution feasibility consideration becomes critical for CBF applications, particularly in autonomous driving that involves complicated safety objectives in the presence of varied driving behaviors.

The scenario studied here is the highway on-ramp merging problem, where the goal is to enable the AV to merge with human-driven cars safely and efficiently. We aim to achieve a multiple-vehicle adaptive merging strategy with safety assurance. Our \textbf{main contributions} are: 1) a bi-level optimization-based control framework with chance-constrained CBF constraints to enable AVs to achieve consistent safe ramp merging behavior; 2) a novel extension of CBF-based safe control constraints to a probabilistic setting for stochastic system dynamics with formally provable safety guarantee; 3) a theoretical analysis to discuss solution feasibility guarantee and design factors reflecting different vehicle behaviors.

\begin{figure}
    \centering
    \includegraphics[width = 0.5\linewidth]{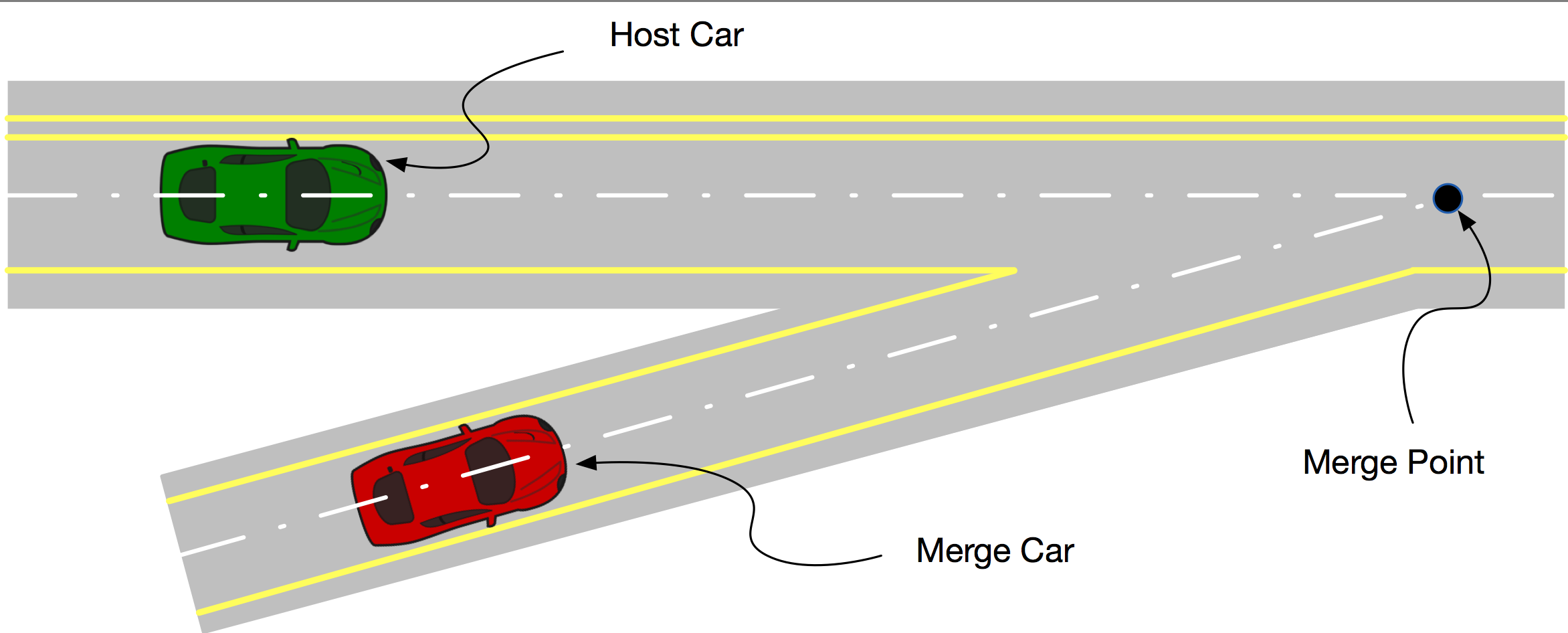}
    \caption{
    \label{scenario}
    Ramp merging scenario (Dong et al. \cite{Dong2017:PGM}). The ego vehicle (green) is an autonomous vehicle, running on the main road; the merging vehicle (red) is a human driven car, running on the ramp. The ego vehicle and the merging vehicle are interchangeable.} 
\end{figure}

\section{Related Work}

In the area of behavior planning and control of autonomous driving, learning-based methods are widely used. Dong et al. \cite{Dong2017:PGM} proposed a Probabilistic Graphical Model-based method to help the ego vehicle decide whether to yield or not in ramp merging, followed by distance-keeping control. Chen \cite{Chen-2018-106042} presents a framework to learn a maneuver controller for autonomous vehicles, which uses LSTM-based methods and deep reinforcement learning for behavior prediction and learning. Nishitani et al. \cite{9197559} introduced a vehicle controller using deep reinforcement learning to improve the merging efficiency while tracking the expected vehicle speed. However, learning-based methods cannot provide a provably correct safety guarantee.

The CBF-based method was initially proposed by Wieland and Allower \cite{inproceedings} in 2007 to describe an admissible control space that renders forward invariance of a safe set. Aaron et al. \cite{7040372} extended CBF to a minimally restrictive setting and applied it to the ACC and lane keeping problem.
CBF-related methods have recently become increasingly popular in different control scenarios. Notomista et al. \cite{9196757}, Zeng et al. \cite{zeng2020safetycritical} and Son et al. \cite{9029446} proposed CBF-based control methods specifically for the two-car racing scenario. Due to the special property of the problem setup, conservativeness is minimized in control effort in order to attain the strongest racing performance. Therefore, as discussed in \cite{9196757}, the methods are not applicable to everyday traffic scenarios, where balance between aggressiveness and conservativeness needs to be maintained for the sake of safety and efficiency.

The validity of using CBF-based methods to achieve safe control has been proved by  \cite{9196757}, \cite{zeng2020safetycritical}, \cite{9029446}, \cite{8917473} and Ames et al. \cite{7782377} in their work. However, none of them gave an explicit quantitative analysis of the solution feasibility condition. \cite{7857061} mentions that the solution feasibility can be guaranteed by assuming that, in the worst case, making all robots decelerate to zero velocity immediately at the next time step can always prevent collision, and therefore the feasible solution space will always be non-empty. However, a more principled scheme with explicit theoretical grounding is desirable to automatically decide whether the vehicle needs to apply full braking before it is too late. Xiao et al. \cite{xiaocdc2020} used machine learning techniques to identify the feasibility of deterministic CBF constraints under fixed parameterization. In this paper, we consider adaptive parameterization of chance-constrained probabilistic CBF constraints to identify as well as modify the feasibility to achieve desired merging behaviors.

To address model uncertainty, several works
(\cite{luo2020multi}, \cite{9196757} and \cite{nikolay2020l4dc}) proposed to employ the CBF approach with noisy system dynamics. However, their works either assumed that the uncertainty is bounded, which could limit the probability distribution, or could not provide a general consistent solution feasibility guarantee. On the other hand, integrating CBF with Model Predictive Control (MPC) is also a common planning method.  Zeng et al. \cite{zeng2020safetycritical} and Son et al. \cite{9029446} introduced MPC-based safety-critical control. However, these works fail to take different driving behavior styles into account and are not able to consider various safe driving strategies.
In this paper, we propose a bi-level optimization-based control framework with assured safety to account for different driving styles and a feasibility guarantee.

\section{Method}
\label{method}

\subsection{Background on Control Barrier Functions}

A Control Barrier Functions (CBF) \cite{ames2019control} is used to define an admissible control space for safety assurance of dynamical systems.
One of its important properties is its forward-invariance guarantee of a desired safety set. Consider the following nonlinear system in control affine form:
\begin{equation}\label{eq:nonlinear}\footnotesize
    \dot x = f(x)+g(x)u
\end{equation}
where $x\in \mathcal{X}\subset \mathbb{R}^n$ and $u\in\mathcal{U}\subset \mathbb{R}^m$ are the system state and control input with $f$ and $g$ assumed to be locally Lipschitz continuous.
A desired safety set $x\in\mathcal{H}$ can be denoted by the following safety function: 
\begin{equation}\label{eq:safeset_general}\footnotesize
\mathcal{H} =\{x \in \mathbb{R}^n : h(x)\geq 0\}
\end{equation}
Thus the control barrier function for the system to remain in the safety set can be defined as follows \cite{ames2019control}:
\begin{definition}
(Control Barrier Function) Given a dynamical system (\ref{eq:nonlinear}) and the set $\mathcal{H}$ defined in (\ref{eq:safeset_general}) with a continuously differentiable function $h:\mathbb{R}^n\rightarrow \mathbb{R}$, then $h$ is a control barrier function (CBF) if there exists an extended class $\mathcal{K}_{\infty}$ function for all $x\in \mathcal{X}$ such that 
\begin{equation}\label{eq:cbf_def}\footnotesize
    \sup_{u\in\mathcal{U}} \ \{L_f h(x)+L_g h(x) u\}\geq -\kappa \big(h(x)\big)
\end{equation}
\end{definition}
\noindent
where $\dot{h}(x,u)=L_f h(x)+L_g h(x) u$ with $L_f h, L_g h$ as the Lie derivatives of $h$ along the vector fields $f$ and $g$.
Similar to \cite{7857061}, in this paper we use the particular choice of extended class $\mathcal{K}_{\infty}$ function with the form as $\kappa (h(x))=\alpha h(x)$ where $\alpha\geq 0$ is a CBF design parameter controlling system behaviors near the boundary of $h(x)=0$. Hence, the admissible control space in Eq.~\ref{eq:cbf_def} can be redefined as 
\begin{equation}\label{eq:cbf}\footnotesize
    \mathcal{B}(x)=\{u\in\mathcal{U}:\dot{h}(x,u) + \alpha h(x)\geq 0\; \}
\end{equation}
It is proved in \cite{ames2019control} that any controller $u\in\mathcal{B}(x)$ will render the safe state set $\mathcal{H}$ forward-invariant, i.e., if the system (\ref{eq:nonlinear}) starts inside the set $\mathcal{H}$ with $x(t=0)\in \mathcal{H}$, then it implies $x(t)\in\mathcal{H}$ for all $t>0$ under controller $u\in\mathcal{B}(x)$.

In this paper, we consider the particular choice of pairwise vehicle safety function $h^s_{em}(x)$, safety set $\mathcal{H}^s$, and admissible safe control space $\mathcal{B}^s(x)$ as follows.
\begin{equation}\label{eq:multi_safety}\footnotesize
\begin{split}
    \mathcal{H}^s&=\{x\in\mathcal{X}: \; h^s_{em}(x) = ||x_e-x_m||^2-R_{safe}^2\geq 0, \;\forall m\} \\
    \mathcal{B}^s(x)&=\{u\in\mathcal{U}: \; \dot{h}^s_{em}(x,u)+ \alpha h^s_{em}(x)\geq 0, \;\forall m\}
\end{split}
\end{equation}
where $x_e,x_m$ are the positions of ego vehicle $e$ and each merging vehicle $m$ with $R_{safe}\in\mathbb{R}$ as the minimum allowed safety distance between pairwise vehicles.

\subsection{Problem Statement}

In this section, the problem formulation in the ramp merging scenario is introduced (Fig. \ref{scenario}). The goal is to control the ego vehicle (host vehicle) on the main road to merge safely with the human-driven vehicles (merge cars) on the ramp with motion uncertainty. The system dynamics of a vehicle can be described by double integrators as follows, since acceleration plays a key role in the safety considerations. 

\begin{equation} \footnotesize
\begin{split}
    \dot{X} &=\begin{bmatrix}
    \dot{x}\\
    \dot{v}
    \end{bmatrix}
    =\begin{bmatrix}
    0_{2\times2}\; I_{2\times2}\\
    0_{2\times 2} \;0_{2\times 2}
    \end{bmatrix}
    \begin{bmatrix}
    x \\
    v
    \end{bmatrix}
    + \begin{bmatrix}
    0_{2\times2}\; I_{2\times2}\\
    I_{2\times2}\; 0_{2\times2}
    \end{bmatrix}\begin{bmatrix}
    u\\ \epsilon
    \end{bmatrix} \\
\end{split}
\label{dynamics}
\end{equation}
where $x\in\mathcal{X}\subset\mathbb{R}^2,v\in \mathbb{R}^2$ are the position and linear velocity of each car respectively and $u\in\mathbb{R}^2$ represents the acceleration control input. 
$\epsilon \sim \mathcal{N}(\hat{\epsilon}, \Sigma)$ is a random Gaussian variable with known mean $\hat{\epsilon}\in\mathbb{R}^2$ and variance $\Sigma\in\mathbb{R}^{2\times2}$, representing the uncertainty in each vehicle's motion.
We assume the human-driven merging vehicle's velocity $v_m$ and the motion uncertainty distribution of $\epsilon_m$ are known to the ego vehicle per time step with $u_m=0$. To create a nonlinear car-like vehicle model from the double integrators formulation, we can use a  kinematics mapping method similar to that in \cite{wang2019game}, which extends single integrator dynamics to car-like robots.

\subsubsection{\textbf{Quadratic Programming Problem Formulation}}

While performing safe merging with human vehicles, the ego vehicle is expected to maintain task efficiency, passing the merging point as fast as possible. Therefore, the objective function can be formulated as a quadratic programming problem for the ego vehicle with the control input $u_e$.
\begin{equation}\footnotesize
\begin{split}
      &\min_{u_e\in\mathcal{U}_e} || u_e-\Bar{u}||^2 \\
    s.t \quad & {U}_{min} \leq  u_e \leq {U}_{max} \\
     &\Pr\Big(\dot{h}^s_{em}(x, u) +\alpha h^s_{em}(x)\geq 0\Big) \geq \eta,\qquad \forall m
\end{split}
\label{qp}
\end{equation}
where $\Bar{u}$ is the nominal expected acceleration for the ego vehicle to follow, and $U_{max}$ and $U_{min}$ are the ego vehicle's maximum and minimum allowed acceleration. We assume $\Bar{u}$ is computed by a higher-level planner, for example, a behavior planner. $R_{s}$ is the minimum allowed distance between two vehicles to avoid collision for safety. 
Different from most existing CBF work with deterministic perfect model information \cite{7857061, celi2019deconfliction}, the stochastic model in Eq.~\ref{dynamics} leads to infinite support of $\dot{h}^s_{em}$ and hence we consider the chance-constrained optimization problem to accommodate uncertainty with $\eta\in(0,1)$ as the desired confidence of probabilistic safety. $\text{Pr}(\cdot)$ denotes the probability of a condition to be true. We employ the chance constraints over vehicle controller $u_e$ to ensure the resulting lower-bounded probability of vehicles being collision-free. This is due to the fact that $\Pr\Big(\dot{h}^s_{em}(x, u) +\alpha h^s_{em}(x)\geq 0\Big) \geq \eta\implies \Pr\Big(h^s_{em}(x)\geq 0\Big) \geq \eta$ given the forward invariance set theory in a deterministic setting in Eq.~\ref{eq:cbf}: $\dot{h}(x, u) +\alpha h(x)\geq 0\implies h(x)\geq 0$ as proved in \cite{ames2019control}.

\subsubsection{\textbf{Bi-level Optimization Problem Formulation}}

In most prior CBF work, the CBF design parameter $\alpha$ in Eq.~\ref{eq:cbf} that determines a particular safe behavior of the robots is often pre-defined and remains fixed, e.g. $\alpha=0$ indicates overly restrictive robot motions for a non-decreasing $h^s_{em}(x)$ while a larger $\alpha$ could yield a more permissive control space. However, a fixed $\alpha$ could make Eq.~\ref{qp} not solvable under certain circumstances \cite{7857061}, i.e., an empty set of $\mathcal{B}^s(x)$, and hence no longer ensure forward-invariant safety. One of the main contributions of our work is formulating the original problem Eq.~\ref{qp} as the following bi-level optimization process with two layers: one for optimization over $u_e$, and the other one for optimization over $\alpha$ for feasibility guarantee.

\begin{equation}\footnotesize
\begin{split}
&\min_{u_e\in\mathcal{U}_e, \alpha\in\mathcal{A}} || u_e-\Bar{u}||^2 \\
    s.t \quad & {U}_{min} \leq  u_e \leq {U}_{max} \\
     &\Pr\Big(\dot{h}^s_{em}(x, u) +\alpha h^s_{em}(x)\geq 0\Big) \geq \eta,\qquad \forall m\\
     &\alpha = \argmin_{\alpha \in \mathcal{A}}\{||\alpha - \Bar{\alpha}||^2\}
\end{split}
\label{bi-level}
\end{equation}
where $\mathcal{A}$ is the feasible set for $\alpha$ that will be proved to ensure solution feasibility of $u_e$. $\Bar{\alpha}$ is a nominal value from the user to specify the desired conservativeness of the safe behavior. The feasible set $\mathcal{A}$ changes over time and to ensure solution feasibility of $u_e$, the goal is to ensure the set $\mathcal{A}$ is consistently non-empty so that we can always find an $\alpha$ that causes a nonempty set of $u_e$ (if it exists) to satisfy the safety constraint. Details will be covered in the following sections. 

\subsection{Active and Feasible Condition of CBF in a Probabilistic Setting}

This section will first present a novel approach to CBF with probabilistic safety consideration under uncertainty and discuss the feasibility analysis with the CBF constraints. 
\begin{theorem}
Given a stochastic dynamical system in Eq. \ref{dynamics} and a confidence level $\eta\in(0,1)$, the following admissible control space $\mathcal{B}^s_{\eta}(x)$ ensures a chance-constrained safety condition in Eq. \ref{bi-level} for the ego vehicle with each merging car $m$.
\begin{equation}\label{eq:prsbc}\footnotesize
    \mathcal{B}_{\eta}^s(x)=\{u_e\in\mathcal{U}_e: \; A_{em}u_e\geq b_{em}, \quad \forall m\}
\end{equation}
\end{theorem}
\begin{proof}
First, consider the CBF constraint in Eq. \ref{eq:multi_safety} and by substituting Eq. \ref{dynamics}, We have:
\begin{equation} \footnotesize
\begin{split}
     &\dot{h}^s_{em}(x,u) +\alpha h^s_{em}(x) \geq 0 \\
 \implies& 2\Delta x_{em}^T\Delta \epsilon_{em} \geq -2 \Delta x_{em}^T(\Delta v_{em}+u_e\Delta t) - \alpha h^s_{em}(x) 
\end{split}
\label{eq5}
\end{equation}
where $\Delta x_{em} = x_e-x_m, \Delta v_{em}=v_e-v_m, \Delta \epsilon_{em}=\epsilon_e-\epsilon_m\sim \mathcal{N}(\Delta \hat{\epsilon}_{em}, \Delta \Sigma_{em})$ for ego vehicle $e$ and each merging vehicle $m$. To tackle the probabilistic version with Pr$(\dot{h}^s_{em}(x) +\alpha h^s_{em}(x)\geq 0)\geq \eta$, we consider Eq.~\ref{eq5} as a chance constraint regarding the Gaussian random variable $\Delta \epsilon_{em}$. From \cite{blackmore2011chance}, a general chance constraint problem takes the following form for an inequality $a^T c \leq b,  a \sim \mathcal{N}(\Bar{a}, \Sigma)$ and can be transformed to a deterministic constraint as 
\begin{equation}\label{eq:chance}\footnotesize
\begin{split}
        &\Pr(a^T c \leq b) = \Phi(\frac{b-\Bar{a}^T c}{\sqrt{c^T\Sigma c}})\\
        \implies &\Pr(a^T c \leq b) \geq \eta \Leftrightarrow b-\Bar{a}^Tc\geq\Phi^{-1}(\eta)||\Sigma^{1/2}c||^2
\end{split}
\end{equation}
where $\Phi^{-1}(\cdot)$ is the inverse cumulative distribution function (CDF) of the standard zero-mean Gaussian distribution with unit variance.
Hence we reorganize Eq.~\ref{eq5} into the form of Eq.~\ref{eq:chance} with $a=\Delta \epsilon_{em}, c=-2\Delta x_{em}, b =2\Delta x_{em}^T (\Delta v_{em}+u_e\Delta t) + \alpha h^s_{em}(x)$ and eventually get:
\begin{equation}\footnotesize
\begin{split}
        A_{em}u_e \leq b_{em},\quad A_{em}\in\mathbb{R}^{1\times 2}, b_{em}\in\mathbb{R}
\end{split}
\end{equation}
with
\begin{equation}\footnotesize
    \begin{split}
         A_{em} =& -2\Delta x_{em}^T\Delta t \\
         b_{em} =& 2\Delta x_{em}^T (\Delta v_{em}+\Delta \hat{\epsilon}_{em})
        +\alpha h^s_{em}(x) \\
        &-\Phi^{-1}(\eta)\sqrt{\Delta x_{em}^T \Delta \Sigma_{em}  \Delta x_{em}}
    \end{split}
    \label{ab}
\end{equation}
where $\Delta t$ is a time unit and we derived control constraints for pairwise chance-constrained safety between ego vehicle and each merging vehicle $m$. This concludes the proof. \qed
\end{proof}
Next, we will discuss the feasibility analysis of Eq.~\ref{bi-level} with safety control constraints in Eq.~\ref{eq:prsbc} and bounded control constraints. In particular, we will present how to ensure non-emptiness of $\alpha$ set $\mathcal{A}$ for non-emptiness of set $u_e\in \mathcal{B}_\eta^s(x)$ that preserves the forward-invariant safety.
Given Eq. \ref{ab}, the set $\mathcal{A}$ feasibility analysis is decomposed into two situations based on the positiveness of $A_{em}$. The active and feasible conditions of CBF depend on the overlap set between the CBF solution set and the bounded control constraints shown in Fig. \ref{feasibleset}. To simplify the discussion, here we assume $u_e=R_{\theta}a_e\in\mathbb{R}^2$ is determined by the ego vehicle's linear acceleration $a_e\in\mathbb{R}$ along the ramp and the rotation matrix $R_{\theta}\in SO(2 )$ by the road geometry. Thus we reformulate Eq.~\ref{eq:prsbc} by $A_{em}(R_{\theta}a_e)\leq b_{em}$ and redefine $A_{em}=A_{em}R_{\theta}\in\mathbb{R},u_e = a_e\in\mathbb{R}$.

\begin{figure}
    \centering
    \includegraphics[width = 0.6\linewidth]{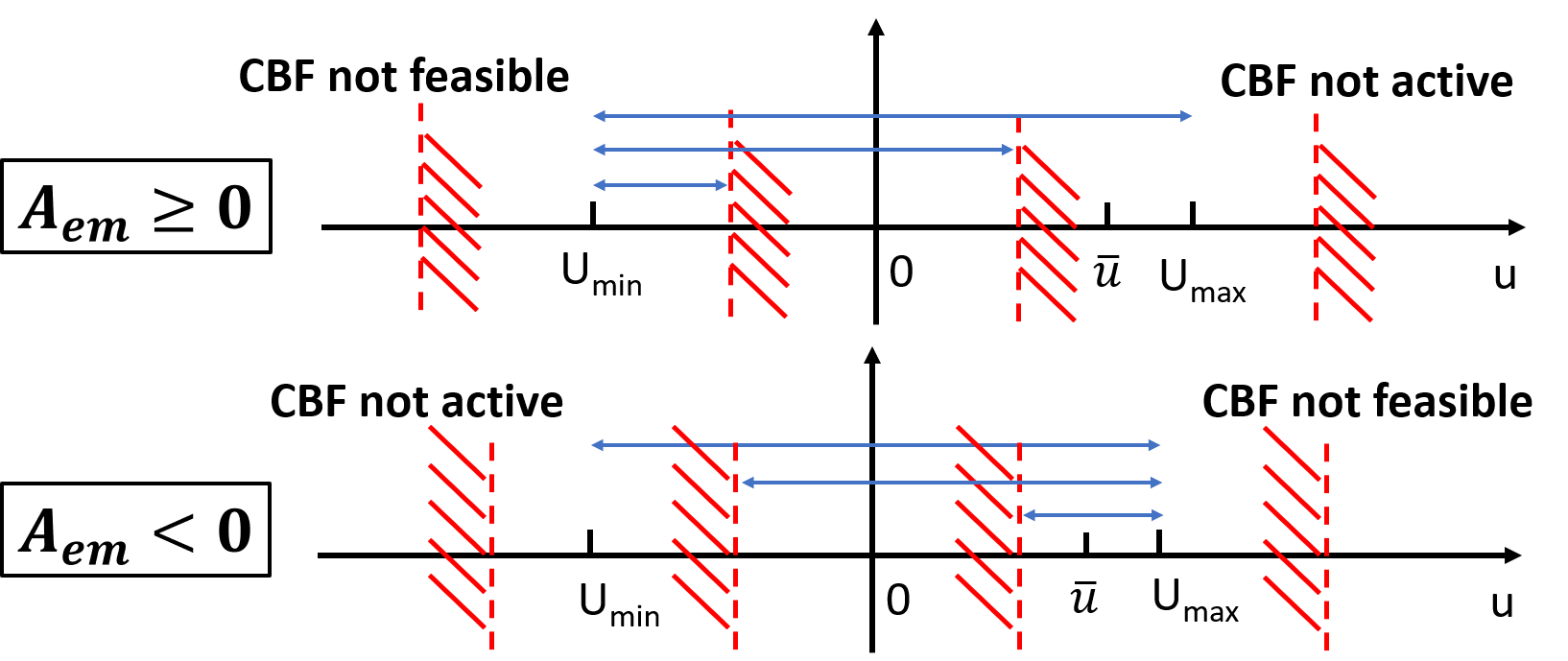}
    \caption{
    \label{feasibleset}
   Illustration of CBF solution feasibility. The red dashed lines stand for the value of $K_1$ and $K_2$ in two cases. Blue intervals represent the reducing solution set while $\alpha$ keeps decreasing.}
\end{figure}

\textbf{Case 1: $A_{em} \geq 0$}.
We have $u_e \leq K_1 = A_{em}^{-1}b_{em}$, which provides an upper bound for $u_e$, meaning the acceleration cannot exceed $K_1$. While the value of $\alpha$ is decreasing and all the other terms remain the same, $K_1$ is also decreasing. Imagine $K_1 = A_{em}^{-1}b_{em}$ as the red vertical dashed line shown in Fig. \ref{feasibleset}, moving from right to left on the 1-D axis of $u_e$. The feasible space resulting from this CBF constraint is anywhere on the left-hand side of the dashed line. The blue intervals represent the feasible solution set while the dashed line is moving. When $\alpha$ is very large, the dashed line $K_1$ is on the right-hand side of the maximum allowed acceleration $U_{max}$. Then the space between $U_{min}$ and $U_{max}$ is feasible, since both the points fall on the left-hand side of the dashed line. The situation changes at a certain point, where $K_1$ overlaps with $U_{max}$, and CBF is activated. The dashed line $K_1$ will keep moving to the left until it overlaps with $U_{min}$, where infeasibility is about to happen due to the empty overlap interval of the two constraints.

\textbf{Case 2: $A_{em} < 0$}.
We have $u_e \geq K_2 = A_{em}^{-1}b_{em}$, which provides a lower bound for $u_e$, meaning the acceleration must be greater than $K_2$. While the value of $\alpha$ is decreasing and all the other terms remain the same, $K_2$ is increasing given the fact $A_{em} < 0$. Again, imagine $K_2 = A_{em}^{-1}b_{em}$ as the red vertical dashed line moves from left to right on the 1-D axis of $u_{e}$. The feasible space resulting from this CBF constraint is anywhere on the right-hand side of the dashed line. When $\alpha$ is very large, the dashed line $K_2$ is on the left-hand side of minimum allowed acceleration $U_{min}$. Then the space between $U_{min}$ and $U_{max}$ is feasible, since both the points fall on the left-hand side of the dashed line. The situation changes at a certain point, where $K_2$  overlaps with $U_{min}$. From this point, the CBF constraint becomes active in optimization while $\alpha$ keeps decreasing. The dashed line $K_2$ will keep moving to the right until it overlaps with $U_{max}$, where infeasibility is about to happen due to the empty overlap interval of the two constraints.

In both cases, after the feasible solution set is determined by updated information on $\alpha$, the optimal control input $u_e$ will be chosen based on the objective function definition. Here, the optimal solution is the closest point from $\Bar{u}$ in the feasible solution set. 

In conclusion, the boundary conditions of the CBF are:
\begin{equation}\footnotesize
    \begin{split}
       A_{em} \geq 0: \quad  &\alpha^m_{feasible} = M_m U_{min}+N_m, \alpha_{active} = M_m U_{max}+N_m \\
      A_{em} < 0: \quad &\alpha^m_{feasible} = M_m U_{max}+N_m, \alpha_{active} = M_m U_{min}+N_m
    \end{split}
\end{equation}
where
\begin{equation}\footnotesize
    \begin{split}
        M_m &= \frac{A_{em}}{h^s_{em}}, N_m = \frac{T_m}{h^s_{em}}\\
       T_m &=- 2\Delta x_{em}^T (\Delta v_{em}+\Delta \hat{\epsilon}_{em})
        +\Phi^{-1}(\eta)\sqrt{\Delta x_{em}^T \Delta \Sigma_{em}  \Delta x_{em}}
    \end{split}
\end{equation}
Generally, the larger $\alpha$ is, the more admissive action space the vehicle will have. To better understand the boundary conditions we derived, intuitively, when the ego vehicle approaches the merging vehicle from behind, the CBF constraint will not be activated until it accelerates and the distance between the two vehicles decreases rapidly. When the ego vehicle drives in front of the merging vehicle, the CBF constraint is activated if the relative speed between the two vehicles decreases rapidly, leading to shorter relative distance. The ego vehicle must maintain at least a certain velocity to prevent the distance from continuing to decrease. The boundary conditions can be seen as a kind of manipulation of the original kinematics constraint, in the form of shrink and shift.

\subsection{Consistent Solution Feasibility Guarantee}
In the previous section, the relationship between $\alpha$ and the solution feasibility was analyzed, and explicit feasible conditions on $\alpha$ were given. Here, a Safe Adaptive Algorithm (Algorithm 1) is introduced for guaranteed solution feasibility. For time steps 1 to N, at each time step $t$, $u_e^t$ is calculated through the first-layer optimization. Then given states of both vehicles, $\alpha_{fea}^{t+1}$ is calculated to ensure the feasible solution set $\mathcal{B}_{\eta}^s(x)$ is non-empty at $t+1$. The second-layer optimization is performed and $\alpha$ is updated at each iteration. The advantage of this algorithm, compared to fixing the $\alpha$ value, is that it provides a dynamic solution feasibility guarantee at run time.

\begin{remark}The problem can still be infeasible with our proposed method if the initial conditions make it impossible to ensure safety, e.g. the ego vehicle is driving too fast, and it's already too late to avoid collision, and no matter what $\alpha$ we choose, Eq. \ref{eq:cbf} can never be satisfied. However, the proposed method does guarantee solution feasibility as long as such a solution exists.\end{remark}

\begin{algorithm}\footnotesize
\caption{Safe Adaptive Merging Algorithm}
\begin{algorithmic}
\REQUIRE $\Delta x_{em}, \Delta v_{em}, \Delta t, R_{safe}, a_m, \Bar{\alpha}$
\ENSURE $\alpha,u_e$
\FOR{$t=1:N$}
\STATE compute $A_{em}^t,b_{em}^t$
\STATE$u_e^t = \argmin_{u_e} ||u_e^t-\Bar{u}||^2$
\STATE compute $A_{em}^{t+1}, M_{em}^{t+1},N_{em}^{t+1}$ via forward kinematics
\IF{$A^{t+1} \geq 0$}
\STATE $\alpha_{fea}^{t+1} = M_{em}^{t+1}U_{min}+N_{em}^{t+1}$
\ELSE
\STATE $\alpha_{fea}^{t+1} = M_{em}^{t+1}U_{max}+N_{em}^{t+1}$
\ENDIF
\STATE $\alpha^{t+1} = \argmin_\alpha ||\alpha - \max(\Bar{\alpha},\alpha_{fea}^{t+1})||^2,\forall m$
\ENDFOR
\end{algorithmic}
\end{algorithm}

\section{Experiment \& Discussion}
\label{experiment}

\subsection{Validity test}
To prove that the proposed method is valid, we conduct experiments against one merging vehicle, with randomly generated ego vehicle initial conditions, including position and velocity and desired driving strategy $\alpha$, and observe the resulting collision rate. The results are shown in Fig. \ref{validity}. The black dashed line stands for the minimum allowed safety distance $R_{safe}$, which is set to be 8m. The confidence level $\eta$ is set to be 99\%. From Fig. \ref{validity}, it is observed that all 400 trials keep the minimum distance as required, and the collision rate is 0\%. The two different kinds of curve shapes correspond to two merging results: asymptotically approaching $R_{safe}$ indicates merging after the merging vehicle and increasing Euclidean distance indicates merging in front of the merging vehicle. In conclusion, no matter what the ego vehicle initial conditions are, and whatever driving strategy the ego vehicle takes, the proposed method can always ensure safety.

\begin{figure}
    \centering
    \includegraphics[width = 0.75\linewidth]{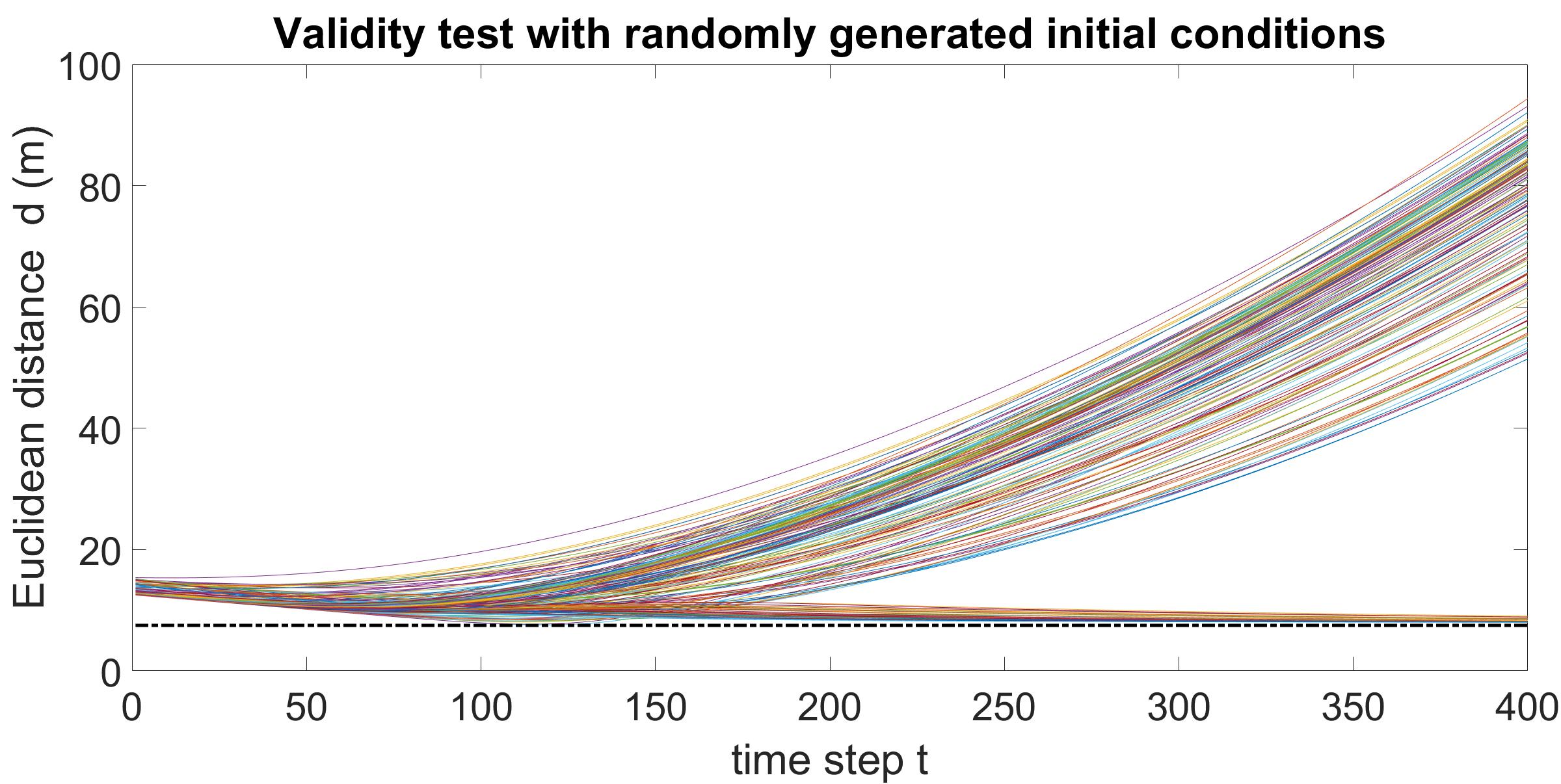}
    \caption{
    \label{validity}
   Validity test of the proposed method.}
\end{figure}

\begin{figure}
    \centering
    \includegraphics[width = 0.75\linewidth]{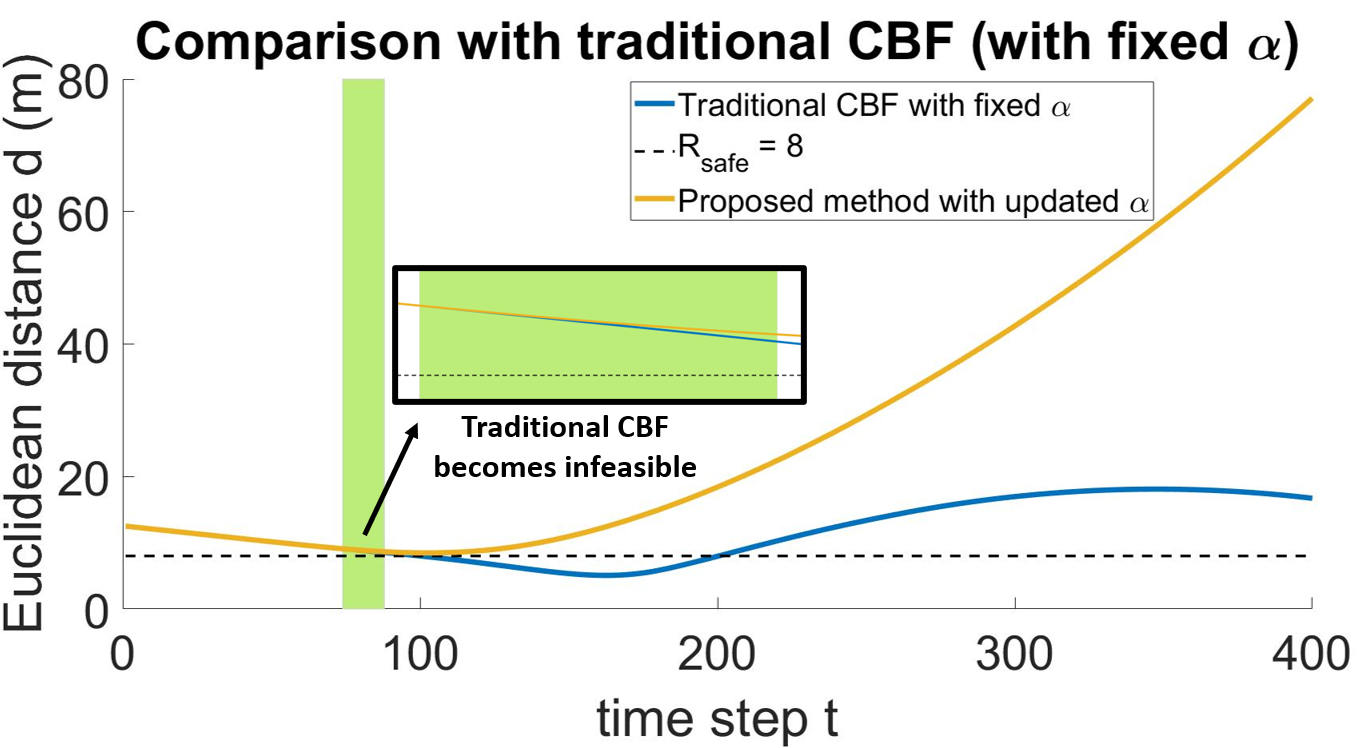}
    \caption{Comparison of the proposed method with traditional CBF with fixed $\alpha$. The green zone indicates the time interval when $\alpha$ is updated in the proposed method to guarantee solution feasibility and therefore safety. The traditional CBF solution with fixed $\alpha$ becomes infeasible starting from this interval and eventually leads to collision with distance smaller than $R_{safe}=8$.}
\end{figure}
\label{fixedalpha}

To better illustrate the advantage of the proposed method, a comparison with traditional CBF with fixed $\alpha$ is made, as shown in Fig. \ref{fixedalpha}. The proposed method updates $\alpha$ in the green zone, while traditional CBF does not, which leads to solution infeasibility from $t=77$ to $t = 194$ and violation of the minimum safety distance requirement. The proposed method maintains solution feasibility consistently and performs the merging safely.

\subsection{Vehicle behavior factors}

\subsubsection{\textbf{Effect of the CBF parameter $\alpha$}}

The choice of the CBF parameter $\alpha$ is a key factor in shaping a vehicle's behavior. As mentioned in Section \ref{method}, the larger $\alpha$ is, the more admissive action space the ego vehicle will have. 

\begin{figure}
    \centering
    \includegraphics[width = 0.8\linewidth]{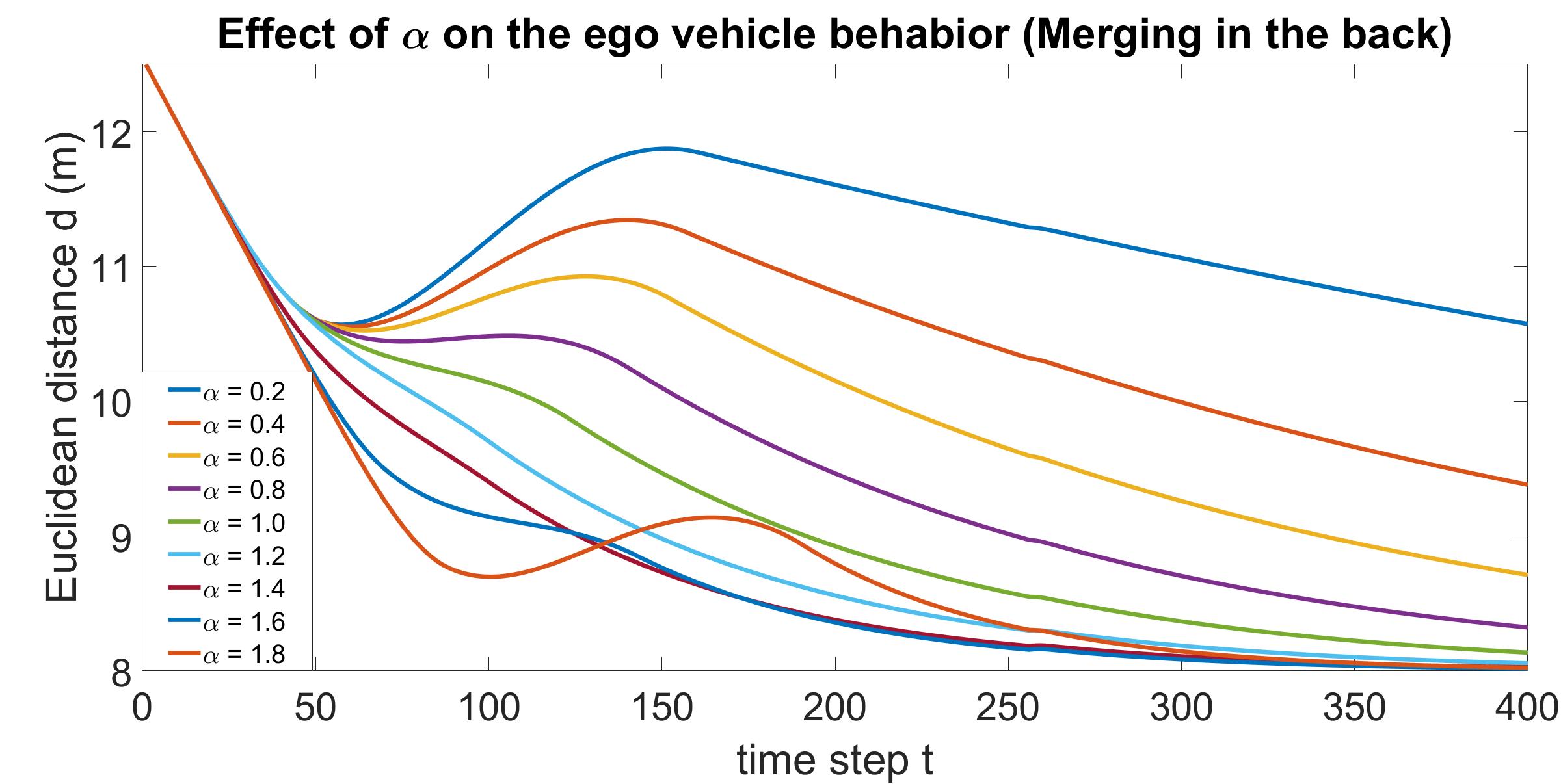}
    \caption{
    \label{alpha-comparison1}
   How $\alpha$ affects the ego vehicle's behavior: While the ego vehicle merges behind, the smaller $\alpha$ is, the earlier the ego vehicle will brake to keep the distance strictly. Larger $\alpha$ will allow the ego vehicle to approach the merging vehicle more quickly, and to brake as late as possible.}
\end{figure}

\begin{figure}
    \centering
    \includegraphics[width = 0.8\linewidth]{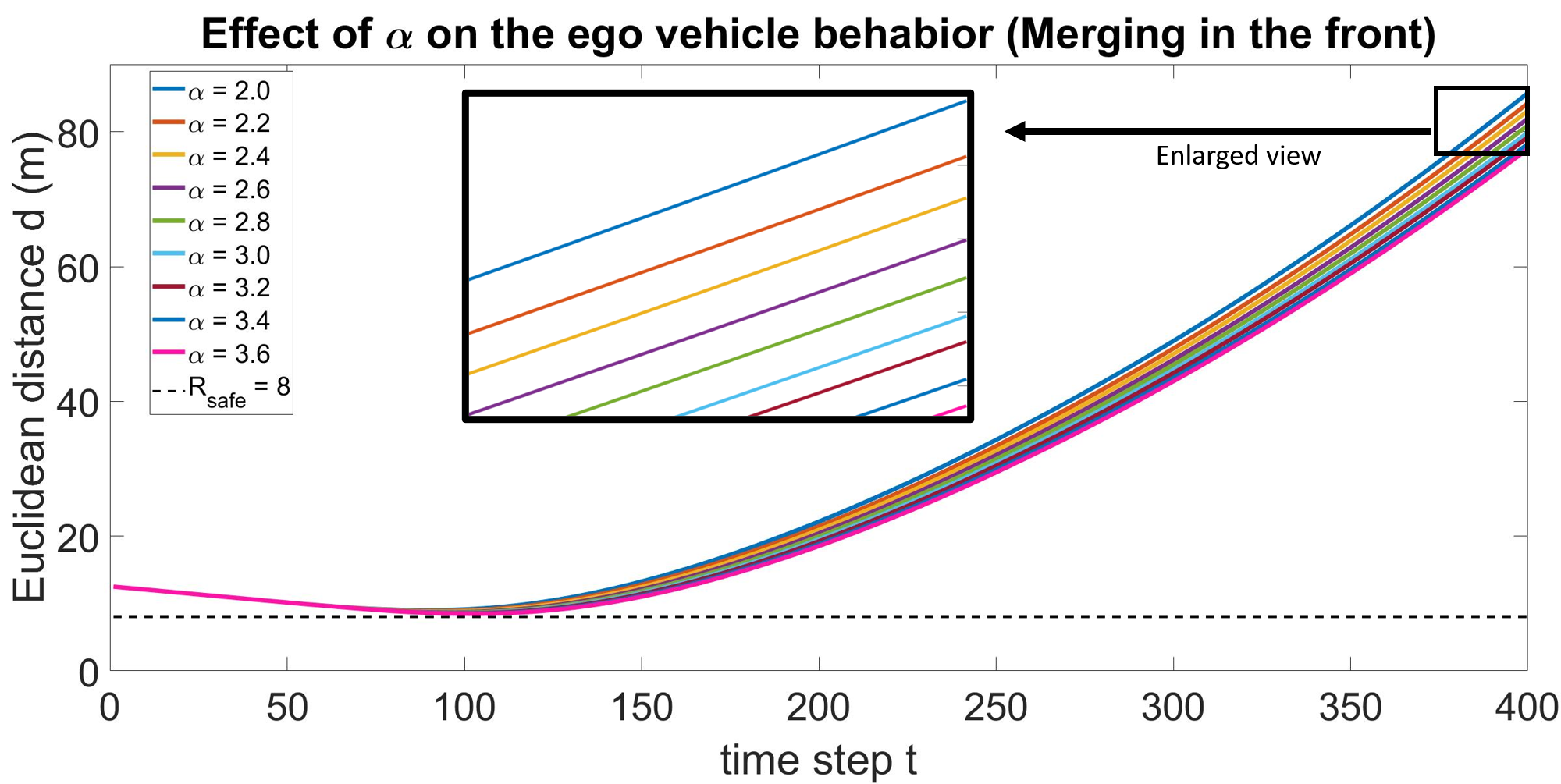}
    \caption{
    \label{alpha-comparison2}
   How $\alpha$ affects the ego vehicle's behavior: While the ego vehicle passes the merging vehicle around $t=90$ in Fig. \ref{alpha-comparison2},  smaller $\alpha$ makes the ego vehicle accelerate as early as possible to prevent getting too close to the merging vehicle in the future, and larger $\alpha$ tends to make the ego drive with $\Bar{u}$ as long as possible while getting closer to the merging vehicle, and only increases acceleration when necessary.}
\end{figure}

To verify this statement, ego vehicle behaviors with different $\alpha$ values are compared. For better visualization effect, we observe ego vehicle merging control with a single merging vehicle, while the initial conditions of the ego vehicle and the profile of the merging vehicle are kept the same to make the comparison fair. The minimum allowed safety distance $R_{safe}$ is set to 8 m. The result is shown in Fig. \ref{alpha-comparison1} and Fig. \ref{alpha-comparison2}. We observe that the smaller $\alpha$ is, the more conservative the driving strategy of the ego vehicle will be, which leaves more tightly bounded action space. 

If we take a closer look at Eq. \ref{ab}, we find that given the same states of both vehicles at time step $t$, decreasing the value of $\alpha$ results in a smaller matrix $b$. Regardless of whether $A$ is positive or negative, $Au_e \leq b$ provides a narrower solution space for $u_e$. This agrees with our observations. Before $t=40$ (Fig. \ref{alpha-comparison1}) and $t=100$ (Fig. \ref{alpha-comparison2}), all trials share exactly the same states, meaning the CBF is not active yet, and therefore the value of $\alpha$ does not make any difference in the ego vehicle's behavior. After those times, the ego vehicle decides to decelerate or accelerate at a certain point. The smaller $\alpha$ is, the earlier the deceleration or acceleration decision is made to avoid collision in future steps. In other words, the larger $\alpha$ is, the more aggressive the driving strategy is for the ego vehicle, and the deceleration as a precautionary action is more and more delayed. This finding also agrees with the experiment results from \cite{7795595}.

\subsubsection{\textbf{Effect of initial conditions}}

The initial conditions also affect the merging behavior of the ego vehicle. Consider the merging control with two merging vehicles, shown in Fig. \ref{case position}. From the high-level decision-making standpoint, the ego vehicle has three merging options: merging in front of $m_2$, merging in between $m_1$ and $m_2$, and merging behind $m_1$. Since the actions of the merging vehicles are not controllable by the ego vehicle, whether the ego vehicle can freely choose to merge into any of the three slots depends on the initial conditions, including relative distance and speed, as well as the nominal acceleration $\Bar{u}$ set for the ego vehicle to follow and the minimum allowed safety distance $R_{safe}$. 

Intuitively, if $R_{safe}$ is set to be very large, the ego vehicle has to keep far enough from both of the merging vehicles, and that makes it difficult for it to squeeze into the gap between the two merging vehicles without breaking the minimum safety distance requirement.

The relative distance and relative speed also matter. Together with $\Bar{u}$, they decide the reachability set for the ego vehicle, which is the set of positions the ego vehicle can achieve. We take a look at two specific cases for detailed illustration. The experimental results are shown in Fig. \ref{different_strategy}. As the initial condition, $v_{m_2} > v_{m_1}$. The minimum allowed safety distance $R_{safe}$ is set to 5 m. The ego vehicle's initial positions for the two cases are shown in Fig. \ref{case position}.
\begin{figure}
    \centering
    \includegraphics[width = 0.55\linewidth]{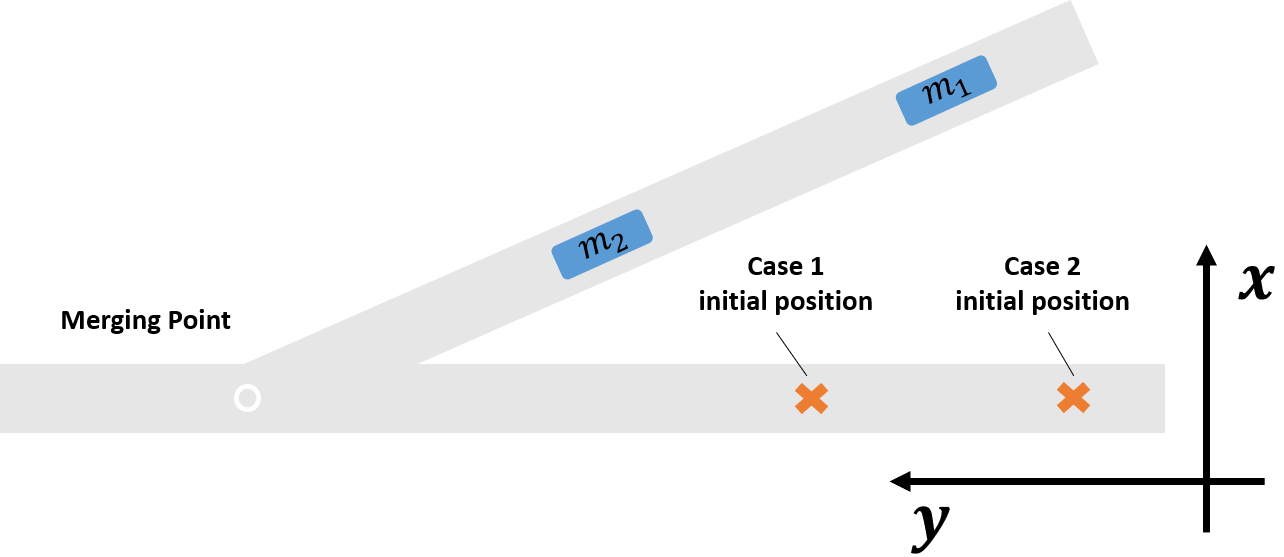}
    \caption{
    \label{case position}
    Illustration of the ego vehicle's initial positions in two cases, where $m_1,m_2$ represent the first and the second merging vehicle.}
\end{figure}

\begin{figure}
    \centering
    \includegraphics[width = 0.9\linewidth]{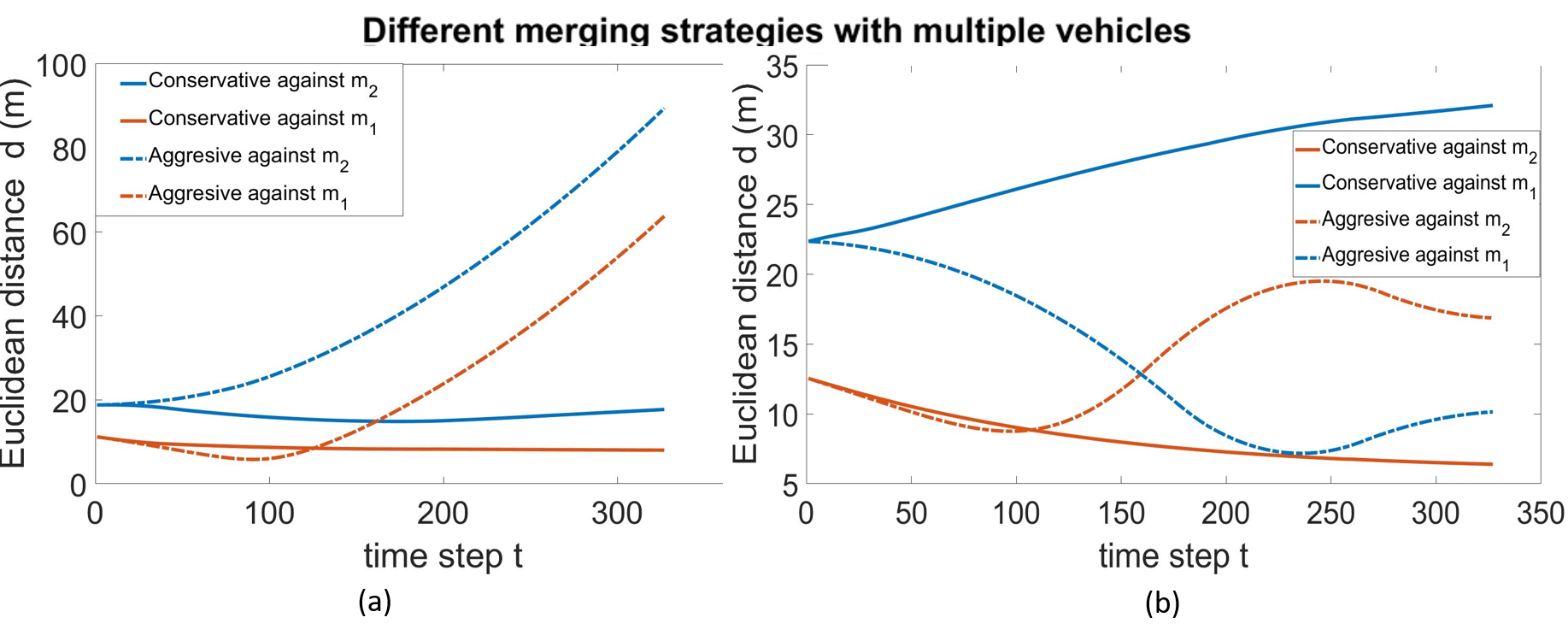}
    \caption{
    \label{different_strategy}
    Comparison of different merging strategies for case 1 (a) and case 2 (b). The x-axis is the time step t and the y-axis is the Euclidean distance between the ego vehicle and each of the merging vehicles. Solid lines stand for the conservative merging strategy with $\alpha = 1$, and the dashed lines stand for the aggressive merging strategy with $\alpha = 15$. The behavior of the two merging vehicles remains the same in the two cases. where $m_1$ and $m_2$ have constant velocity.}
\end{figure}

 In case 1, $v_e = v_{m1}+2$. While performing conservative merging, with $\alpha$ = 1, the ego vehicle merges in between $m_1$ and $m_2$ safely. The ego vehicle keeps the distance to both of the merging vehicles without much change. While performing aggressive merging, with $\alpha =15$, the ego vehicle merges in front of $m_2$ safely. The ego vehicle accelerates obviously and completes the merging around $t=95$, and the distance to both merging vehicles increases after that.

In case 2, $v_e = v_{m_1}-2$. While performing conservative merging, with $\alpha$ = 1, the ego vehicle merges after $m_1$ safely. The ego vehicle accelerates while maintaining the required safety distance. While performing aggressive merging, with $\alpha =15$, the ego vehicle merges in between the two merging vehicles safely. The ego vehicle accelerates to pass $m_1$ and merges in between $m_1$ and $m_2$. It keeps approaching $m_2$ until it decreases its acceleration to meet the future safety guarantee.


\section{CONCLUSIONS}
\label{conclusion}
We present a novel adaptive merging control algorithm for autonomous driving vehicles in highway scenarios with probabilistic safety guarantee. Simulations with different conditions are used to demonstrate the power of CBF in applying different driving strategies to the ego vehicle via a single parameter $\alpha$. The problem is uniquely formulated as a chance-constrained Control Barrier Function-based bi-level optimization, which provides a theoretically consistent solution feasibility analysis with explicit bounds on the CBF parameter $\alpha$, and thus makes the future use and applications of CBF more flexible without concerns about solution infeasibility under certain circumstances. The proposed feasibility-guaranteed CBF-based method not only works in the single highway merging problem, but presents a paradigm for application to other problems, along with formulation and analysis guidelines.  
In future work, we plan to combine the proposed framework with learning-based methods that use real-world datasets to realize safe control with a data-driven approach to determine the appropriate strategy.

\bibliography{ICRA}{}
\bibliographystyle{IEEEtran}

\end{document}